\documentclass[10pt, conference, letterpaper]{IEEEtran}
\makeatletter

\newcommand{\Rmnum}[1]{\expandafter\@slowromancap\romannumeral #1@}
\makeatother
\usepackage{graphicx}
\graphicspath{{../eps/}}
\graphicspath{{../pdf/}{../jpeg/}}
\DeclareGraphicsExtensions{.pdf,.jpeg,.png}
\ifCLASSOPTIONcompsoc
  \usepackage[caption=false,font=normalsize,labelfont=sf,textfont=sf]{subfig}
\else
  \usepackage[caption=false,font=footnotesize]{subfig}
\fi
\usepackage{cite}
\usepackage{amssymb,amsmath,amsfonts}
\usepackage{amsthm}
\usepackage{mathrsfs}
\usepackage{epstopdf}

\usepackage{float}

\usepackage{cases}
\graphicspath{{./Figures/}}
\usepackage{algorithmic}
\usepackage{algorithm}

\usepackage[T1]{fontenc}
\usepackage[utf8]{inputenc}
\usepackage{authblk}

\usepackage{url, xcolor}

\newcommand{\black}{\color{black}}

\usepackage{verbatim, booktabs}
\usepackage{amsmath}

\newtheorem{theorem}{Theorem}
\newtheorem{definition}{Definition}
\newtheorem{remark}{Remark}

\usepackage{array}
\usepackage{multirow}
\usepackage{makecell}
\usepackage{blindtext}

\usepackage[compact]{titlesec}         
\titlespacing{\section}{2pt}{2pt}{2pt} 
\AtBeginDocument{
  \setlength\abovedisplayskip{0pt}
  \setlength\belowdisplayskip{0pt}}

\begin{document}

\title{Filtered Randomized Smoothing: A New Defense for Robust Modulation Classification}

\author[]{Wenhan Zhang$^*$~~~Meiyu Zhong$^*$~~~Ravi Tandon~~~Marwan Krunz}

\affil[]{Department of ECE, University of Arizona, Tucson, AZ \authorcr E-mail: \{\textit{wenhanzhang, meiyuzhong, tandonr, krunz}\}@arizona.edu}

\maketitle
\def\thefootnote{*}\footnotetext{Equal Contribution }\def\thefootnote{\arabic{footnote}}
\thispagestyle{plain}
\pagestyle{plain}
\begin{abstract}
Deep Neural Network (DNN) based classifiers have recently been used for the modulation classification of RF signals. These classifiers have shown impressive performance gains relative to conventional methods, however, they are vulnerable to imperceptible (low-power) adversarial attacks. Some of the prominent defense approaches include adversarial training (AT) and randomized smoothing (RS). While AT increases robustness in general, it fails to provide resilience against previously unseen adaptive attacks. Other approaches, such as Randomized Smoothing (RS), which injects noise into the input, address this shortcoming by providing provable certified guarantees against arbitrary attacks, however, they tend to sacrifice accuracy. 
 
In this paper, we study the problem of designing robust DNN-based modulation classifiers that can provide provable defense against arbitrary attacks without significantly sacrificing accuracy. To this end, we first analyze the spectral content of commonly studied attacks on modulation classifiers for the benchmark RadioML dataset. We observe that spectral signatures of un-perturbed RF signals are highly localized, whereas attack signals tend to be spread out in frequency.  To exploit this spectral heterogeneity, we propose Filtered Randomized Smoothing (FRS), a novel defense which combines spectral filtering together with randomized smoothing. FRS can be viewed as a strengthening of RS by leveraging the specificity (spectral Heterogeneity) inherent to the modulation classification problem. In addition to providing an approach to compute the certified accuracy of FRS, we also provide a comprehensive set of simulations on the RadioML dataset to show the effectiveness of FRS and show that it significantly outperforms existing defenses including AT and RS in terms of accuracy on both attacked and benign signals. 

\end{abstract}

\section{Introduction}
In recent years, Deep Neural Network (DNN) based classifiers have emerged as a promising alternative for modulation classification in wireless systems. Leveraging the ability of DNNs to learn complex patterns and features from raw data, DNN-based classifiers offer promising performance in accurately identifying the modulation scheme using in-phase/quadrature (I/Q) samples. However, these DNNs are prone to low-power \textit{imperceptible} attacks, which can be readily generated using adversarial machine learning (AML) based methods \cite{Zhang2021MILCOM, Zhang2024TMLCN, Kim2021TWC, Flowers2020TIFS, Adesina2023CST}. The broadcast nature of the wireless medium makes AML attacks a significant threat and roadblock for widespread deployment of DNN-based classifiers. For instance, an adversary can broadcast such low-power AML perturbations to degrade the signal identification accuracy of legitimate users and spoof the wireless operator.

To build robust modulation classifiers against AML attacks, recent work has developed several defense mechanisms, often adapted from the existing pool of defenses which were designed for generic classifiers.  
Olowononi {\em et al.} \cite{Olowononi2020CST} proposed an encryption mechanism to hide the DNN internal weights, parameters, and training data from an adversary. He {\em et al.} \cite{He2023CST} evaluated Adversarial Training (AT), randomization, defensive distillation, and gradient masking to defend against adversarial attacks. Zhang {\em et al.} \cite{Zhang2024TMLCN} presented adversarial training, autoencoder-based denoising, and classifier ensembling to mitigate the impact of AML attacks. 

 While the above defenses do improve the resilience of DNN classifiers, most of these heuristics ultimately fail to generalize against stronger and previously unseen adaptive attacks. Therefore, a line of work focusing on the notion of \textit{certified defense} has emerged, wherein the classifier must guarantee to maintain consistent predictions within the adversarial attack budget, thereby ensuring robustness. 
 \begin{figure*}
    \centering
    \includegraphics[scale = 0.38]{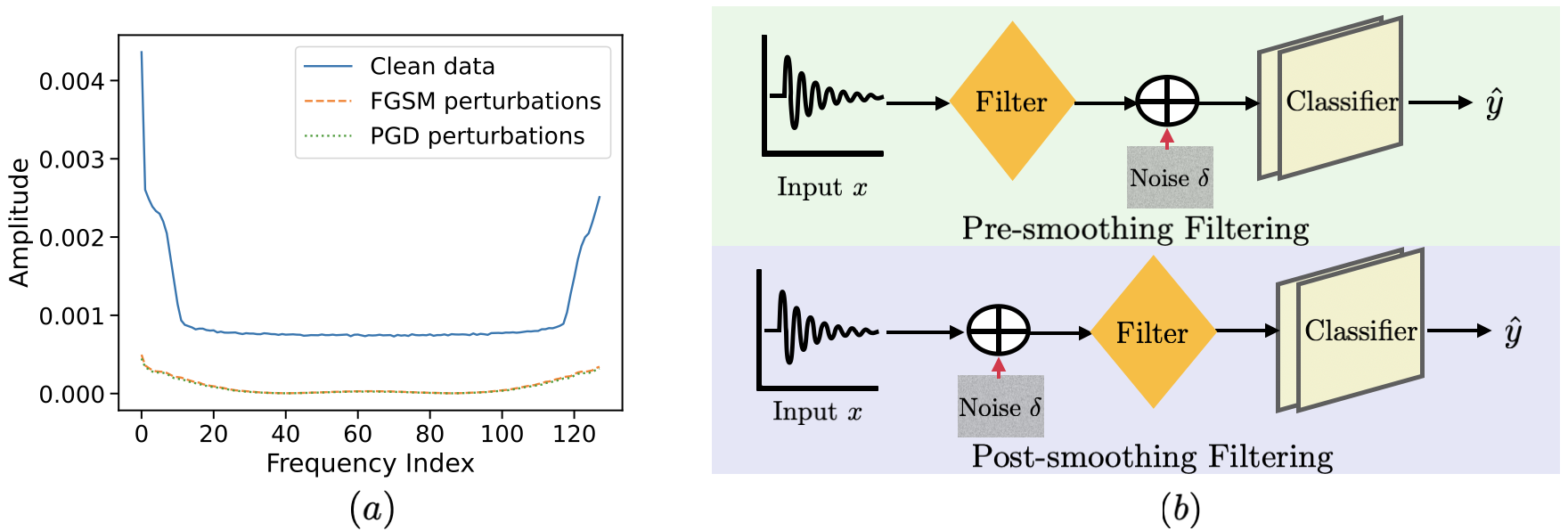}
    \vspace{-0.18 in}
    \caption{(a) Comparison of the frequency content of clean signals versus two attack signals, FGSM- and PGD-based perturbations. The figure shows the amplitude of frequency components: FFT averaged over data at 18 dB. (b) Illustration of filtered randomized smoothing (FRS) defense, with two variations: post-smoothing filtering (Theorem \ref{the: certified_radius}) and pre-smoothing filtering (Theorem \ref{the:filter_rs}).}
    \label{fig:filter_fig}
    \vspace{-0.2 in}
\end{figure*}
One of the most important certified robustness mechanisms, known as Randomized Smoothing (RS), is introduced in \cite{Cohen2019ICML,Zhai2020MACER:, zhong2024splitz}. The main idea of RS is to add multiple independent realizations of noise (e.g., Gaussian noise) to the input; each of which is passed to the classifier. The respective decisions of the noisy input are then combined to make a classification decision; the resulting classifier can be shown to be provably robust within a certified radius (maximum allowable attack budget), which is a function of the noise strength. Research on RS-based certified defense has been explored in several directions:  Zhai et al. \cite{Zhai2020MACER:} introduce a regularization strategy that maximizes the approximate certified radius,  Salman et al. \cite{salman2019provably} integrate adversarial training with smoothed classifiers and Kim et al. \cite{Kim2021TWC} investigate the RS in the wireless domain.


\textit{Spectral Heterogeneity in Clean vs. Attack RF Signals:} The above defense approaches are generic in nature and do not necessarily exploit the structure of RF signals and waveforms. To this end, we conducted an in-depth study of the spectral composition of RF signals drawn from a well-studied benchmark RadioML dataset, as well as AML attacks on these signals. Our  findings led to the following observations: clean (un-attacked) RF signals from the waveform typically tend to concentrate in a low-frequency range. In contrast, the frequencies of natural noise and AML attacks are spread over a wider interval. For instance, Fig. \ref{fig:filter_fig}(a) shows an illustrative example of this phenomenon on signal selected from RadioML dataset: the energy of the clean signal is concentrated below $15$, whereas the spectrum of two common AML attacks, namely Fast Gradient Sign Method (FGSM) \cite{Goodfellow2015ICLR} and Projected Gradient Descent (PGD) \cite{madry2018towards} perturbations are widely spread. 


\noindent \textbf{Overview of Filtered RS \& Contribution:}  To utilize this spectral heterogeneity, we propose Filtered Randomized Smoothing (FRS), a novel defense that combines spectral filtering together with randomized smoothing.
The main idea behind FRS is to filter the input RF signal by attenuating high-frequency components, which serve the role of reducing the contribution of AML attacks, without degrading the contribution of the legitimate RF signal. We combine filtering with randomized smoothing—adding noise either before or after the filter—so we can strengthen the theoretical foundation of the filtered-RS model. Fig. \ref{fig:filter_fig}(b) shows the conceptual illustration of the proposed FRS defense. The main contributions of this paper are summarized next:

\begin{itemize}
  \item We present a comprehensive spectral analysis of a benchmark RadioML modulation classification dataset (both on clean signals as well as adversarial attacks such as FGSM and PGD). We observe that training a modulation classifier on filtered signals alone can achieve $20$\% higher test accuracy than a regularly trained classifier under both FGSM and PGD attacks, on average.
  \item To further enhance the robustness of filter-based mechanisms and achieve certified robustness, we introduce \textit{Filtered Randomized Smoothing (FRS)} with two variations: pre-smoothing filtering and post-smoothing filtering and also provide theoretical results on the certified robustness.
  \item We provide a comprehensive experimental evaluation of FRS and compare it with adversarial training and conventional randomized smoothing. Our experimental findings demonstrate that our proposed Filtered Randomized Smoothing classifier outperforms other models, including those utilizing AT and RS, in terms of certified test accuracy.
\end{itemize}

\section{Problem Statement and Preliminaries}
We represent a modulation classifier through a mapping $\hat y$ = $f(x)$, where the input $x \in \mathbb{R}^{2\times W}$ represents a window of I/Q samples with window size $W$. The first (second) row represents the sequence of I (Q) samples, respectively. The output $\hat y$ represents a probability distribution over $\{1,2,\ldots, K\}$, where $K$ is the number of possible modulation schemes (classes). Our paper aims to make a robust classifier such that $f(x)=f(x+\delta)$ where $\delta$ is the perturbation generated by the adversary and constraint by $\left \|\delta  \right \|_2 \leq \epsilon$ and $\epsilon$ stands for the energy budget for attacks. 
We define the Signal-to-Perturbation Ratio (SPR) as the energy ratio of the received signal and the perturbation: $E(x)/E(\delta)$,  
where $E(x)$ is the average clean signal energy before the additive perturbation.
We next give a brief overview of the two most prominent defense techniques, namely: 1) Adversarial Training (AT) and 2) Randomized Smoothing (RS). 

\textbf{Adversarial Training}
The main idea behind AT~\cite{Goodfellow2015ICLR, Adesina2023CST, Zhang2024TMLCN} is as follows: we start with a base classifier, and generate adversarial attacks on the training data. Subsequently, the training data is augmented with these attacked signals and the classifier is re-trained using the following loss: 
\begin{equation}
    \tilde{L}(x,y;\theta )=\gamma L(x,y;\theta )+(1-\gamma)L(x_{adv},y;\theta ),
\end{equation}
where $\gamma$ controls the balance between benign and adversarial data. In our experiments, we use the default value of $\gamma = 1/2$, as suggested in~\cite{Goodfellow2015ICLR}, and it gives us the best accuracy under both benign and adversarial data. Although AT provides robustness against AML attacks, the main shortcoming is that it relies on the knowledge of Adversarial Examples (AEs) which are created using specific attacks. While classifiers trained using AT perform better against the attacks that were used in AT, however, it has been shown \cite{jia2022adversarial, de2022make} that 
such classifiers are not robust to previously un-seen adaptive attacks.


\textbf{Certified Defense and Randomized Smoothing}
The above lack of generalizability of AT has led to the stronger notion of certified robustness as defined next:
\begin{definition}(Certified Robustness) \label{def: certified robustness}
A (randomized) classifier $f$ satisfies $(\epsilon, \alpha)$ certified robustness if for any input $x$, we have 
\begin{center}
    $\mathbb{P}(f(x) = f(x')) \geq 1- \alpha$, $\forall x'$, such that $ x' = x + \delta, ~~\parallel\delta\parallel_p \leq \epsilon$.
\end{center}
\end{definition}
Certified robustness requires that a classifier's decision remains unchanged in a local neighborhood around any given test input $x$. Specifically, for all inputs $x'$ near $x$, where the distance between $x'$ and $x$ under the $\ell_p$ norm ($\parallel x'-x \parallel_{p}$) is less than or equal to $\epsilon$, the classifier's output should be the same as that for $x$, i.e., $f(x) = f(x')$, with high probability. Therefore, $\epsilon$ is defined as the certified radius, and $(1-\alpha)$ quantifies the confidence level. For the scope of this work, we focus primarily on the $\ell_2$ norm ($p=2$).

\textit{Randomized Smoothing:} RS was introduced and analyzed in \cite{Cohen2019ICML} for achieving certified robustness. Specifically, RS involves taking an arbitrary base classifier (denoted by $f$), and transforms it into a "smooth" classifier, $g$ defined as:
\begin{align}
    g(x) =  \underset{c~\in~\mathcal{Y}}{\text{argmax}}~ \underset{\delta \sim \mathcal{N}(0,\sigma^2 I)}{\mathbb{P}}(f(x + \delta)=c).
\end{align}
Intuitively, for a given input $x$, the function $g(x)$ outputs the class that the base classifier $f$ predicts most frequently within the neighborhood of $x$. Unlike adversarial training, RS provides certified robustness guarantees by adding random noise to the input and taking the majority vote of the base classifier’s outputs over many noisy samples.  This smooth classifier not only retains a desirable property of certified robustness but also offers an easily computable closed-form certified radius $\epsilon$.  While RS provides provable robustness, its solution is very generic and does not exploit the specific characteristics of the modulation classification problem.


 

\section{Spectral Analysis of Adversarial Attacks on Modulation Classification}


\begin{table*}[t]
\vspace{-0.12 in}
\caption{Evaluation of Filter for Each Class with Cut-off Frequency Index $k=5$}
\vspace{-0.1 in}
\begin{center}
\resizebox{0.9\textwidth}{!}{
\begin{tabular}{|c|c|c|c|c|c|c|c|c|c|c|c|c|}

\hline
  \hline
    Metrics &  8PSK  &  AM-DSB &  AM-SSB &  BPSK &  CPFSK &  GFSK &  PAM4 &  QAM16 &  QAM64 &  QPSK &  WBFM & Averaged
 \\
 \hline

    $\eta_{be}$ (dB)& -1.62&  -0.20&  -9.79&  -1.75&  -1.01& 
       -0.29&  -1.65&  -1.62&  -1.63&  -1.59&   -0.21& -1.94

 \\
 \hline
    $\eta_{pe}$ (dB) & -2.59& -1.28& -2.99& -3.83& -2.93&
       -2.13& -1.77& -2.52& -2.51& -3.37& -1.32& -2.47

\\
 \hline

    SPR (dB) & 15.94& 15.65&  8.85& 17.40& 16.54&
       16.41& 16.12& 16.18& 16.19& 16.71& 15.68& 15.61

\\
 \hline

\end{tabular}
}
\end{center}
\label{filter_class}
\vspace{-0.2 in}    
\end{table*}
We analyze the spectral characteristics of the wireless modulation classification dataset, Radio Machine Learning (RML) 2016.a \cite{OShea2018JSTSP}.
To compare these frequency components, we calculate the DFT of data.
In DFT, the component at frequency index $k$ can be expressed as: $X_k = \sum_{n = 0}^{N-1}x_ne^{-j2\pi k n/N}.$
Where $N$ is the number of samples, $n$ is the index for the current sample in the time domain, and $x_n$ is the value of sample $n$. In our dataset, we keep $k$ in the same range as $N$ to calculate the frequency components and index frequencies from 0 to $127$. We use the I/Q data in a complex format and apply Fast Fourier Transform (FFT) to expedite processing. 
\subsection{Butterworth Low-pass Filter}
We consider the Butterworth low-pass filter to have a frequency response flatten in the passband. 
The gain function $G(\cdot)$ and frequency response function $H(\cdot)$ of an $m$th-order Butterworth low-pass filter can be expressed as: $G^2(\omega)=\left | H(j\omega) \right |^2=\frac{{G_0}^2}{1+(\frac{j\omega}{j\omega_c})^{2m}}.$
$\omega_c$ represents the cut-off frequency and $G_0$ denotes the gain at zero frequency. With a larger $m$, the cut-off is sharper, and the filtered waveform experiences more time-domain shifts. 
Therefore, we start with $m=2$
and keep the same filter frequency index by $k$ as in DFT to evaluate the impact of cut-off frequency.




\subsection{Averaged Spectral Behaviour \& Impact of Filter}
Fig. \ref{fig:filter_fig}(a) shows the average spectral behavior of the clean signals as well as the respective values for the attack signals generated using FGSM and PGD attacks. This motivates us to understand the impact on the behavior of clean and attack signals if they are passed through a low-pass filter (such as the Butterworth filter described above). We propose two metrics to evaluate the impact of the filter. The first metric is the post-/pre-filtering power ratio $\eta$, which represents the energy ratio between the passband and unfiltered signal. The second metric is the Signal-to-Perturbation Ratio (SPR), defined as the energy ratio between the benign data and perturbations.

In Fig. \ref{fig:eng_rate_comp}(a)(right), we vary the cut-off frequency of the filter and evaluate $\eta$ for both the benign data and FGSM perturbations. We observe that when the cut-off frequency index $k$ is less than $15$, the post-/pre-filtering power ratio for the benign data ($\eta_{be}$) is higher than that for the perturbations ($\eta_{pe}$). This indicates that the low-pass filter removes more perturbations than benign components with a small $k$. However, when $k$ is greater than 15, the passband ratio for the benign data and perturbations becomes similar, suggesting that the filter has a comparable impact on these two types of data. At $k = 64$, all the signals pass through the filter, resulting in a ratio of $0$ dB.
In Fig. \ref{fig:eng_rate_comp}(a)(left), we evaluate the SPR between the filtered benign data and perturbations when applying the filter with different cut-off frequencies. We evaluate FGSM and PGD attacks with $\epsilon = 0.015$ and $0.03$ as examples. When $k$ is large, the SPR for filtered signals remains the same. In contrast, when $k$ is small, the signal quality with filtering is better than in the unfiltered one. The trend is similar for all four considered attacks, suggesting designing the filter with a small $k$.
\black

\begin{figure*}
    \centering
    \includegraphics[scale = 0.65]{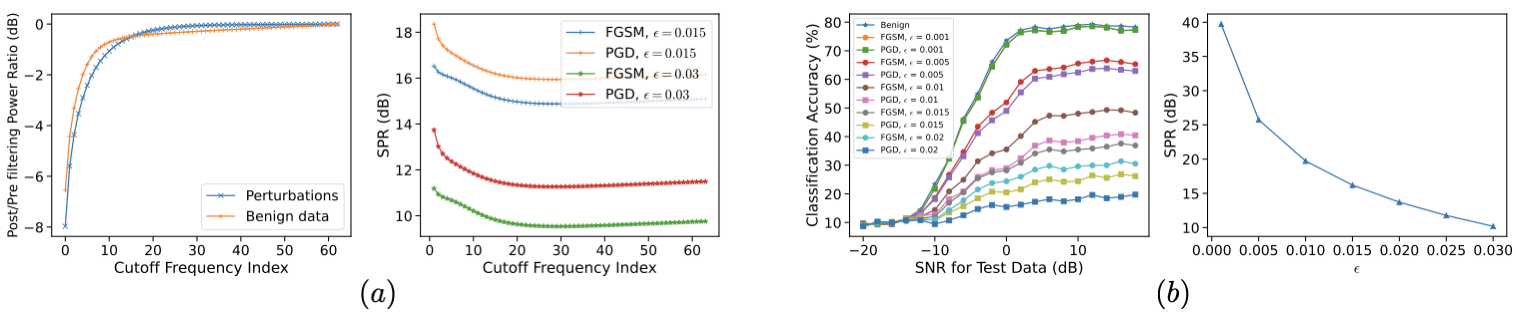}
    \vspace{-0.18 in}
    \caption{(a) Energy rate under different cut-off frequency index: left: Passband signal rate, right: SPR. (b) Classification accuracy under AML attacks : left: Accuracy vs. SNR under attacks of various $\epsilon$, right: SPR vs. $\epsilon$ for FGSM attacks. }
    \label{fig:eng_rate_comp}
    \vspace{-0.2 in}
\end{figure*}

\subsection{Impact of Filtering on Individual Subclasses}
To estimate the impact of the filter on each class of data, we calculate the metrics shown in Table \ref{filter_class}. We observe that $\eta_{be}$ can be around 1 dB higher than $\eta_{pe}$, and the SPR is improved to greater than 15 dB for most classes. However, the filter does not improve $\eta$ and SPR for AM-SSB. This suggests that the symmetric nature of our filter design may compromise the single-sideband modulated signal. 
To enhance the robustness of filter-based mechanisms, we introduce Filtered Randomized Smoothing in the next section, where we propose two types of Filtered Randomized Smoothing a) pre-smoothing filtering and b) post-smoothing filtering. 
\section{ Filtered Randomized Smoothing}\label{sec: filter_RS}




In this section, we introduce the details of the certifying process of the smooth classifier. We first illustrate the theoretical results of the robust classifier based on randomization smoothing (RS). Then, we provide the robustness guarantee of the combination of the filter and RS. We finally discuss how we implement the certifying process in the radio machine learning (RML) dataset. 

We assume the base classifier \( f \) identifies the most probable class \( c_A \) with probability \( p_A \), and the second most likely class with probability \( p_B \). We denote $\underline{p_A}$ as the lower bound of $p_A$, $\overline{p_B}$ as the upper bound of $p_B$. 
To integrate a filter with randomized smoothing, there are two distinct approaches: a) Pre-smoothing filtering (denoted by Pre-FRS): applying the filter before Gaussian noise augmentation, and b) Post-smoothing filtering (denoted by Post-FRS): injecting the filter after Gaussian noise augmentation. Randomized smoothing mechanisms are known for their scalability across various black-box mechanisms. Thus, introducing the filter post-noise augmentation does not compromise the theoretical assurances of randomized smoothing, as outlined in Theorem \ref{the: certified_radius}. For approach a), it is necessary to first estimate the Lipschitz constant of the filter, followed by deriving the corresponding certified radius as shown in Theorem \ref{the:filter_rs}.

We first show the robustness guarantee of the randomized smoothing (Post-smoothing filtering) as follows:

\begin{theorem}\cite{Cohen2019ICML}\label{the: certified_radius} 
\textit{(Post-smoothing filtering)}
Let \( f : \mathbb{R}^n \to \mathcal{Y}\) represent any deterministic or stochastic function, with \( \epsilon \sim \mathcal{N}(0, \sigma^2 I) \). Defining \( g \) as per Equation (1) and with \( c_A \) specified, if \( \underline{p_A}, \overline{p_B} \in [0, 1] \) meet the criteria

\begin{equation}
 \mathbb{P}(f(x + \epsilon) = c_A) \geq \underline{p_A} \geq \overline{p_B} \geq \underset{c \neq c_A}{\text{max}}~\mathbb{P}(f(x + \epsilon) = c). 
\end{equation}

Then \( g(x + \delta) = c_A\) for all \( \|\delta\|_2 < R \), where:

\begin{equation}
 R_{\text{Post-FRS}} = \frac{\sigma}{2}(\Phi^{-1}(\underline{p_A}) - \Phi^{-1}(\overline{p_B})) 
 \vspace{-0.05 in}
\end{equation}
where \( \Phi^{-1} \) denotes the inverse of the standard Gaussian CDF.
\vspace{-0.2 in}
\end{theorem}

\begin{remark}
    The above result does not presuppose any specific characteristics about $f$, highlighting its scalability to large models, whose properties may be difficult to estimate. In addition, the value of the certified radius $R$ increases with a higher noise level $\sigma$. Note that a high value of $\sigma$ may sacrifice the models' utility at the same time. Therefore, there exists a trade-off between robustness and accuracy, which can be navigated by tuning the noise parameter.
\end{remark}

We next present the theoretical robustness guarantee of Pre-smoothing filtering in the following Theorem:
\begin{theorem}\label{the:filter_rs}
\textit{(Pre-smoothing filtering)}
Let us denote $L_{lip}$ as the Lipschitz constant of the filter, and $R_{rs}$ as the certified radius of the randomized smoothing classifier with confidence $1-\alpha$. Therefore, the certified radius of the filter (pre-noise) smoothing mechanism with confidence $1-\alpha$ is:
\begin{align}
    R_{\text{Pre-FRS}} = \frac{R_{\text{rs}}}{L_{\text{lip}}}.
\end{align}
\end{theorem}

\begin{proof}
\vspace{-5pt}
    To simplify the notation, we use $f_{\text{filter}}, f_{\text{rs}}$ to denote the filter and randomized smoothing (RS) classifiers respectively. From the definition of Lipschitz constant of $f_{\text{filter}}$, we note that
    \begin{align}
        \parallel f_{\text{filter}}(x) - f_{\text{filter}}(x') \parallel_2 \leq L_{\text{lip}}\parallel x - x'\parallel_2. \label{eq: th2_10}
    \end{align}
    In addition, we are given the RS classifier $f_{\text{rs}}$ has a certified radius $R_{\text{rs}}$ with probability $(1-\alpha)$. Since the output of the filter is an input to the RS classifier,   therefore, to ensure the certified robustness of the pre-smoothing filtering classifier ($f_{\text{rs}} \circ f_{\text{filter}} $), we require:
    \begin{align}
        \parallel f_{\text{filter}}(x) - f_{\text{filter}}(x') \parallel_2 \leq L_{\text{lip}}\parallel x - x'\parallel_2 \leq R_{\text{rs}}. \label{eq: th2_12}
    \end{align} 
    From the above inequality, we can arrive at the claim that for all $(x,x')$, such that 
   $\parallel x - x'\parallel_2 \leq \frac{R_{\text{rs}}}{L_{\text{lip}}}$, the decision of $f_{\text{rs}} \circ f_{\text{filter}} $ will remain the same with probability $(1-\alpha)$. Therefore, the certified radius of $f_{\text{rs}} \circ f_{\text{filter}} $ is given by $\frac{R_{\text{rs}}}{L_{\text{lip}}}$ completing the proof of the Theorem. 
\end{proof}
\begin{remark}
    The Lipschitz constant not only quantifies the stability of the filter in pre-smoothing filtering approaches but also helps balance the trade-off between robustness and accuracy. Specifically, a smaller Lipschitz constant can enhance certified robustness, but this often comes at the cost of reduced certified test accuracy.
\end{remark}

\noindent \textbf{Inference and Certification for Modulation Classification}. 
During the filtered-RS certifying process, the smoothed classifier makes predictions on the noisy samples for each filtered input. The smoothed classifier's output is then the class that has the majority vote among all these noisy samples. After getting the prediction of the top class $c_A$ and the runner-up class $c_B$ of the smoothed classifier, we calculate the radius of robustness, which is the size of the perturbation the classifier can tolerate without changing its prediction. This is done by estimating the probability of the predicting top class $c_A$ ($c_B$, respectively), namely $p_A$ ($p_B$, respectively).
To this end, each test-filtered input has its corresponding prediction and certified radius. Note that if the input is filtered before injecting the noise (pre-smoothing filtering), we need to estimate the Lipschitz constant of the filter. The certified radius of the pre-smoothing filtered RS is stated in Theorem \ref{the:filter_rs}. For the input filtered after injecting the noise, the certified radius is the same as the certified radius of randomized smoothing (Theorem \ref{the: certified_radius}).  

\begin{figure*}
    \centering
    \includegraphics[scale=0.68]{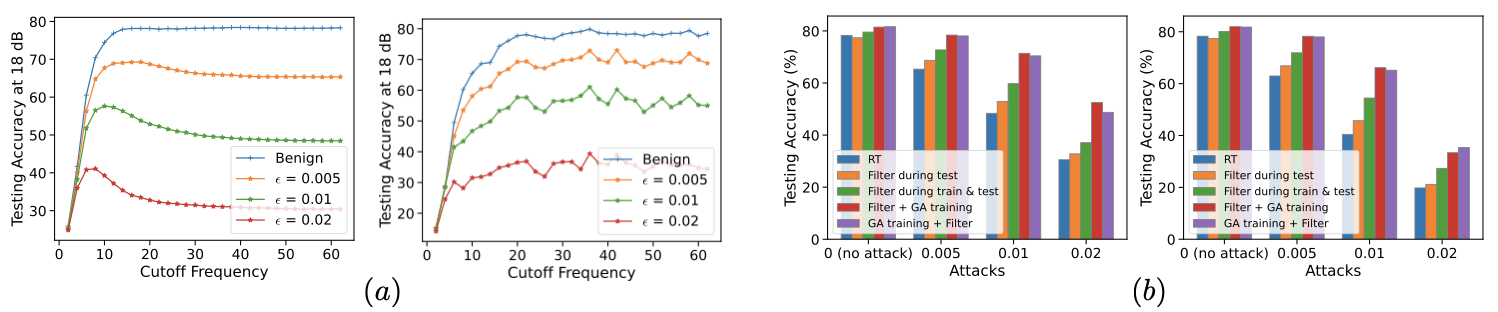}
    \vspace{-0.18 in}
    \caption{(a) Impact of cut-off frequency when applying the filter-based defense: left: During testing, right: during both training and testing. (b)Evaluation of the filter-based defense: left Tested under FGSM attacks, right: tested under PGD attacks.}
    \label{fig:cutoff_comp}
    \vspace{-0.13 in}
\end{figure*}
\begin{figure*}[t]
    \centering
    \includegraphics[scale=0.4]{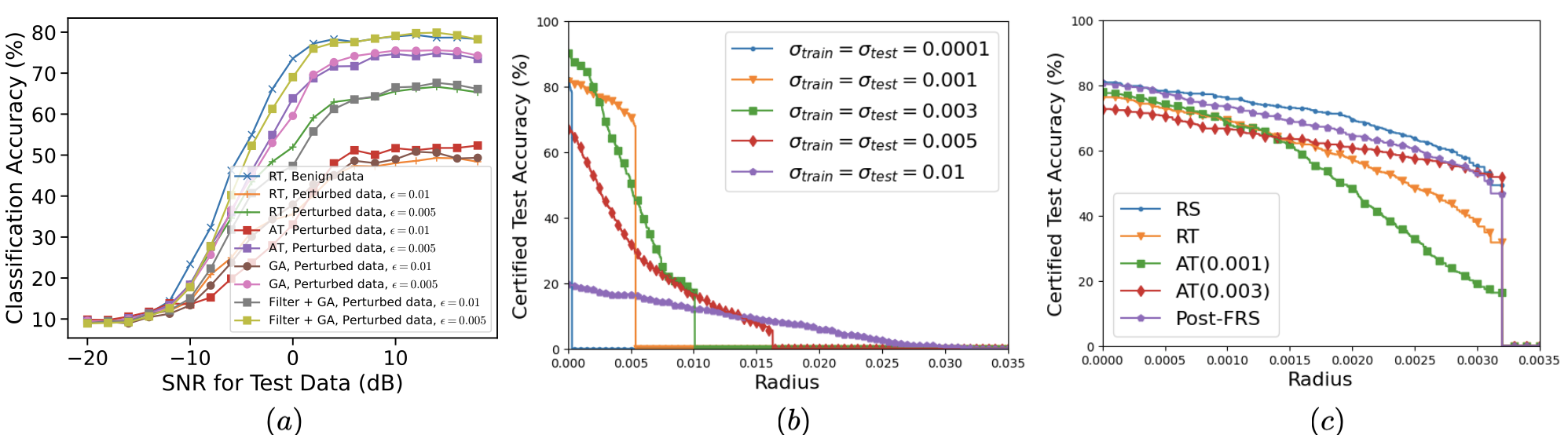}
    \vspace{-0.11 in}
    \caption{(a) Comparison of different defenses during testing. (b) The trade-off between Robustness and Accuracy under different values of variance ($\sigma_{train} = \sigma_{test}$). We observe that when $\sigma_{test} = 0.001$, our classifier can achieve a better trade-off between robustness and accuracy. (c) The trade-off between Robustness and Accuracy under different models with $\sigma_{test} = 0.001$, where RS represents that the model is trained with Gaussian noise, RT represents regular training, AT($\epsilon$) denotes that the model is trained using AT with attack budget of $\epsilon$, RS + Filter represents that we filter the noise samples during training (post-noise filtering).}
    \label{final_compare}
    \vspace{-0.2 in}
\end{figure*}

\section{Performance Evaluation}
In this section, we present our experimental findings. We first detail the dataset, classifier, and corresponding experimental settings. Subsequently, we discuss the impact of the filter during the testing phase. Finally, we present the certified test results for the Filtered Randomized Smoothing classifier.

\subsection{Dataset, Classifier, and Attack Descriptions}
We consider the RML 2016.10a dataset \cite{OShea2018JSTSP} with the corresponding
modulation classifier (VT-CNN2) proposed by O'Shea et al. . This dataset includes noisy I/Q samples for 11 modulation schemes: 8PSK, BPSK, QPSK, QAM16, QAM64, CPFSK, GFSK, PAM4, WBFM, AM-DSB, and AM-SSB. Each modulation scheme is represented in 1,000 windows of samples for each given SNR, with the SNR varying from -20 dB to 18 dB in steps of 2 dB, resulting in a total of 220,000 windows of samples. The RML dataset has a window size of 128 samples (I/Q pairs), with a stride of 64. To reduce the impact of resampling, we use 50\% of the data for training, 5\% for validation and early stopping, and 45\% for testing. 
We compare Filtered Randomized Smoothing (FRS) with the following baselines: a) RT: classifier obtained via regular training classifier, b) AT: classifier obtained via adversarial training, c) GA: classifier with Gaussian noise augmentation (adding Gaussian noise during the training process), and d) RS: randomized smoothing classifier (adding Gaussian noise during both the training, inference and certification).

We demonstrate in Fig. \ref{fig:eng_rate_comp}(b)(left) that the VT-CNN2 classifier \cite{OShea2018JSTSP} exhibits lower accuracy under $\ell_2$ normed attacks compared to benign testing. With a higher $\epsilon$, the attack becomes stronger, resulting in lower classification accuracy. To compare the energy of the perturbation with the benign signal, we use SPR as the measurement. As shown in Fig. \ref{fig:eng_rate_comp}(b)(right), the larger $\epsilon$ is expected to result in a lower SPR. 
\vspace{-0.05 in}
\subsection{Selection of Frequency Parameters}
\vspace{-0.05 in}
When we fix $m$, the cut-off frequency $\omega_c$ plays the most important role in filter design. We consider the two application scenarios: (i) applying the filter as a plug-in unit during training, and (ii) applying the filter during both training and testing. In Fig. \ref{fig:cutoff_comp}(a)(left), we observe that the defender's accuracy for (i) starts at a low point and increases with the \textit{cut-off} frequency index $k$. 
This occurs because the low-pass filter allows only a small part of the frequency component to pass through when $k$ is low. Consequently, the classifier cannot accurately identify the waveform with limited information in the filtered signal. 
We also observe that accuracy starts dropping after $k$ exceeds a certain threshold. 
This is because the filter allows more frequency components in perturbations to pass through, resulting in a saturated accuracy similar to the case without applying the filter.
These trends are similar when tested under FGSM attacks with $\epsilon \in \{0.005, 0.01, 0.02\}$. In Fig. \ref{fig:cutoff_comp}(a)(right), we illustrate the accuracy for (ii) under different attacks as we increase $k$, which saturates after a certain threshold. Therefore, we select $k$ to be $20$ since it gives us the highest accuracy under all attacks and represents the turning point presented in Fig. \ref{fig:cutoff_comp}(a)(right). Another observation is that the achievable accuracy of (ii) is higher than (i). 
\vspace{-0.05 in}
\subsection{Impact of Filtering}
\vspace{-0.05 in}
\textbf{Filtering and Gaussian Noise Augmentation during Testing}.
We also evaluate the impact of different enhancements in the filter-based defense under both FGSM and PGD attacks with different values of $\epsilon$ (attack budget). 
In addition to comparing the filter applied in different phases, we combine Gaussian randomization with the filter-based approach. Since the order of adding noise can impact certified robustness, we consider both adding noise before and after filtering.
As shown in Fig. \ref{final_compare}(a), adopting the filter-based defense improves the defender's accuracy under all considered attacks.
By combining filter design with Gaussian randomization, the defender's accuracy gets further improved. Comparing the red bar with blue one, the proposed approach remains
effective even when $\epsilon$ takes larger values.  On average, our proposed defense increases the accuracy by $19.37 \%$ for FGSM and $18.21 \%$ for PGD. This indicates that our approach despite not relying on the type of attacks, can still provide robustness.



When the attacker has a small $\epsilon$ (0.005), both AT and GA can effectively increase the accuracy under attacks. However, when $\epsilon$ is relatively large, these two defense mechanisms lose effectiveness. In contrast, our approach can still significantly enhance the defense accuracy even with a large $\epsilon$. In addition, the proposed filter-based approach outperforms the other two in both regimes of $\epsilon$.


\textbf{Filtered Randomized Smoothing \& Certification}.
We explore the trade-off between  Robustness and Accuracy for the RML dataset. Following previous works \cite{Cohen2019ICML}, we set the confidence parameter $\alpha = 0.001$, e.g., with probability $99.9\%$, the radius returned by $g$ is truly robust. Our certified test accuracy refers \textit{approximate certified test set accuracy} \cite{Cohen2019ICML}, denoted as the proportion of the test dataset that the smooth classifier $g$ correctly identifies (without abstaining) and confirms as robust if a certified radius R of the input $x$ greater than or equal to r (given values, such as 0.01, 0.02, etc.). During certification, we use $10,000$ augmented noise samples to estimate the certified radius. As shown in Fig. \ref{final_compare}(b), we found that $\sigma = 0.001$ achieves better results compared to other values of variance. Therefore, we use $\sigma = 0.001$ as the noise variance.

We now explore the effect of different training mechanisms in the trade-off between robustness and accuracy, as shown in Fig. \ref{final_compare}(c). We can observe that the RS consistently outperforms other models. In addition, we can observe that RS outperforms AT with $\epsilon=0.001$ and $\epsilon=0.003$.
We now study the performance of the filtered RS w.r.t the trade-off between robustness and accuracy as shown in Fig. \ref{final_compare}(c). 
The classifier with Post-FRS performs better than AT classifiers. Overall, Post-FRS achieves relatively high trade-off between robustness and accuracy.

\section{Conclusions}
In this paper, filtered randomized smoothing (FRS), a new defense against adversarial attacks was presented, which combines low-pass filtering and randomized smoothing. We demonstrated that adversarial perturbations exhibit different spectral features than benign data, and applying a low-pass filter can mitigate their impact without significantly degrading signal quality. Combining Gaussian noise-based smoothing with filtering can further enhance classifier accuracy under adversarial attacks. Theoretical results were presented which can be used to compute the certified accuracy of FRS-based classifiers. In addition, extensive experimental results on validating the proposed FRS defense were provided. We presented that FRS outperforms conventional defenses, such as AT, and RS achieving higher certified test accuracy for a wide range of channel conditions and larger attack budgets.

\section{Acknowledgment}
This research was supported in part by NSF (grants 2229386, 1822071, 2100013,  2209951, 1651492 and 2317192), by the Broadband Wireless Access $\&$ Applications Center (BWAC), U.S. Department of Energy, Office of Science, Office of Advanced Scientific Computing under Award Number DE-SC-ERKJ422, and by NIH through Award 1R01CA261457-01A1. Any opinions, findings, conclusions, or recommendations expressed in this paper are those of the author(s) and do not necessarily reflect the views of the sponsors.

\bibliographystyle{IEEEtran}
\bibliography{References.bib}

\end{document}